\documentclass[final,1p,times]{elsarticle}
\usepackage{amssymb}
\usepackage{amsmath}
\usepackage{amsthm}
\usepackage{lineno}
\usepackage{multirow}
\usepackage{bbding}
\usepackage{graphicx}
\usepackage{subfig}
\usepackage{colortbl}
\usepackage{pifont}
\usepackage[table]{xcolor} 
\usepackage{booktabs}
\newtheorem{theorem}{Theorem}[section] 
\usepackage[breaklinks=true,colorlinks,bookmarks=True]{hyperref}
\journal{Neurocomputing}
\begin{document}
\begin{frontmatter}
\title{CMaP-SAM: Contraction Mapping Prior for SAM-driven Few-shot Segmentation}
\author[label1,label2]{Shuai Chen}
\author[label1]{Fanman Meng\corref{mycorrespondingauthor}}
\cortext[mycorrespondingauthor]{Corresponding author}
\ead{fmmeng@uestc.edu.cn}
\author[label2]{Liming Lei}
\author[label1]{Haoran Wei}
\author[label1]{Chenhao Wu}
\author[label1]{Qingbo Wu}
\author[label1]{Linfeng Xu}
\author[label1]{Hongliang Li}
\affiliation[label1]{organization={University of Electronic Science and Technology of China}, city={Chengdu}, postcode={611731}, country={China}}
\affiliation[label2]{organization={TaiHang Laboratory}, city={Chengdu}, postcode={610213}, country={China}}

\begin{abstract}
Few-shot segmentation (FSS) aims to segment new classes using few annotated images. While recent FSS methods have shown considerable improvements by leveraging Segment Anything Model (SAM), they face two critical limitations: insufficient utilization of structural correlations in query images, and significant information loss when converting continuous position priors to discrete point prompts. To address these challenges, we propose CMaP-SAM, a novel framework that introduces contraction mapping theory to optimize position priors for SAM-driven few-shot segmentation. CMaP-SAM consists of three key components: (1) a contraction mapping module that formulates position prior optimization as a Banach contraction mapping with convergence guarantees. This module iteratively refines position priors through pixel-wise structural similarity, generating a converged prior that preserves both semantic guidance from reference images and structural correlations in query images; (2) an adaptive distribution alignment module bridging continuous priors with SAM's binary mask prompt encoder; and (3) a foreground-background decoupled refinement architecture producing accurate final segmentation masks. Extensive experiments demonstrate CMaP-SAM's effectiveness, achieving state-of-the-art performance with 71.1 mIoU on PASCAL-$5^i$ and 56.1 on COCO-$20^i$ datasets. Code is available at \url{https://github.com/Chenfan0206/CMaP-SAM}.
\end{abstract}


\end{frontmatter}

\section{Introduction}
\label{intro}

Image segmentation is the foundational task in computer vision with broad applications in medical diagnosis, remote sensing, and industrial inspection. Despite significant advances in fully supervised segmentation methods~\cite{long2015fully,chen2017deeplab, cheng2021per,SARFormer}, their deployment in real-world scenarios faces two critical limitations: the prohibitive cost of acquiring pixel-level annotations, and the inability to generalize to novel object categories beyond the training set. These challenges are particularly acute in specialized domains where expert knowledge is required for annotation or where data samples are inherently limited, creating a critical bottleneck for practical applications.

Few-shot segmentation (FSS)~\cite{tian2020prior, shi2022dense,min2021hypercorrelation,boudiaf2021few,Wang_2023_CVPR,cheng2023hpa,10109193,dissanayake2025few,wang2025few,su2025few} addresses these limitations by enabling models to segment novel object categories using only a handful of annotated examples. The core idea involves leveraging limited labeled samples in a support set to guide the segmentation of query images. In many FSS methods, position prior~\cite{tian2020prior, shi2022dense} play an important role in directing segmentation. These priors, typically generated through feature matching or prototype learning between support and query images, provide probabilistic spatial representations that indicate potential target regions. By incorporating these position priors, models can more effectively filter out background distractions and focus on regions likely to contain the target object.

\begin{figure}[htbp]
  \centering
  \includegraphics[width=1. \linewidth]{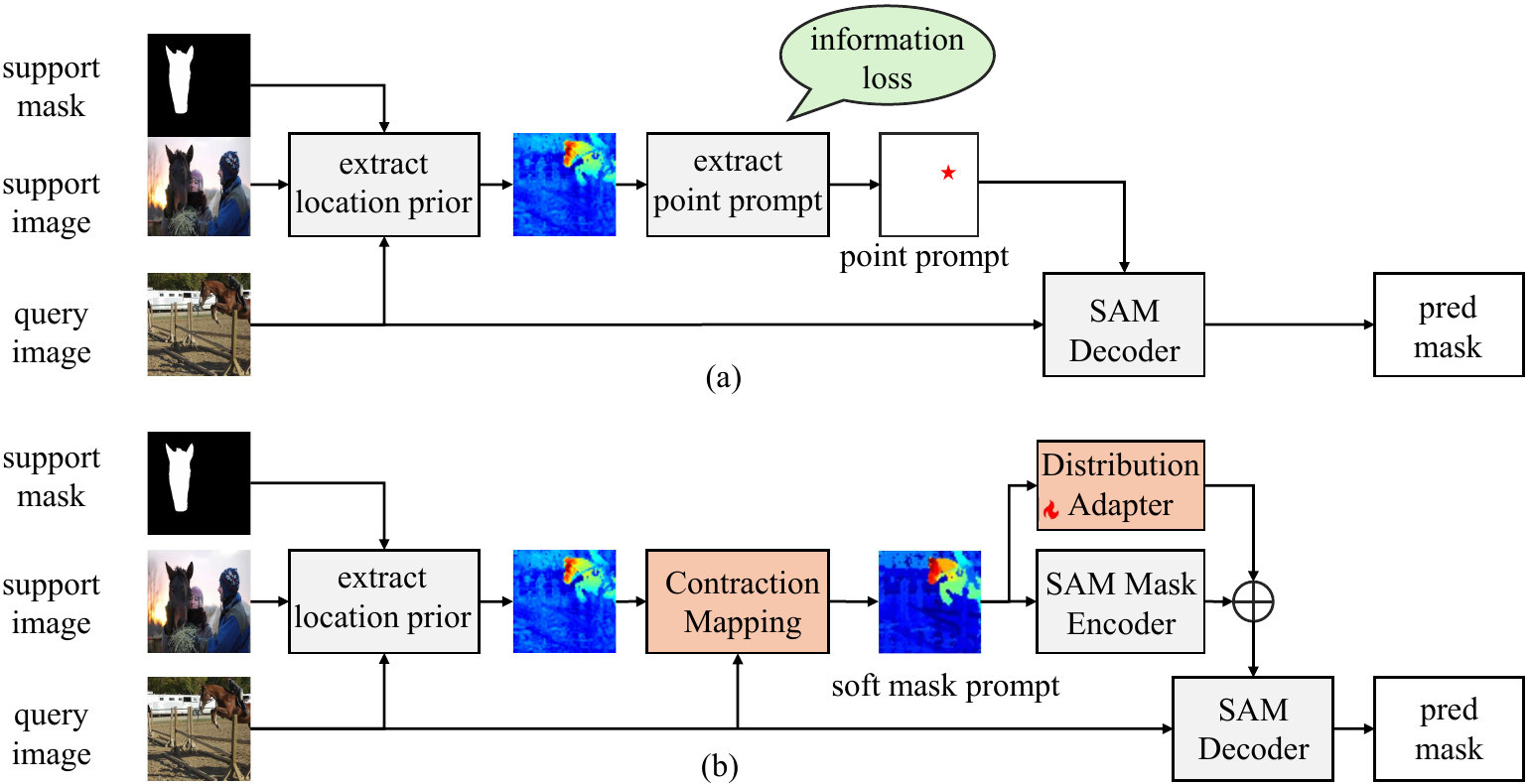}
  \caption{Comparative analysis: (a) Conventional approaches employ ``generate location prior, then extract point prompt'' pipeline, causing information loss. (b) The proposed CMaP-SAM optimizes position priors through contraction mapping theory, preserving both semantic guidance from support images and structural correlations within query images. Additionally, a distribution adapter is implemented to bridges continuous probability distributions with SAM's binary mask prompt encoder, thereby eliminating information loss.}
  \label{fig:motivation}
\end{figure}

The recent introduction of the Segment Anything Model (SAM)~\cite{kirillov2023segment, ravi2024sam} represents a paradigm shift in segmentation capabilities. Through pretraining on massive datasets, SAM demonstrates remarkable zero-shot generalization with its prompt-driven architecture. The integration of SAM and FSS creates a symbiotic relationship with bidirectional benefits. SAM typically requires manual prompts for each test image, an often time-consuming process, but FSS methods can automatically generate visual references as prompts, significantly enhancing SAM's practicality. Conversely, SAM's powerful segmentation capabilities offer substantial performance improvements for FSS on novel categories.

However, effectively integrating SAM with FSS frameworks presents significant technical challenges. First, while FSS methods typically generate continuous probability distributions as position priors, SAM expects discrete prompts (points, boxes, or binary masks) as input. As shown in Fig.~\ref{fig:motivation} (a), current approaches~\cite{feng2024boosting} attempt to bridge this gap employ a sequential ``generate coarse mask, extract discrete points'' strategy, which introduces substantial information loss during the quantization process. The rich uncertainty information encoded in probability distributions, essential for characterizing ambiguous regions, is progressively lost when continuous distributions are compressed into binary masks and subsequently reduced to sparse point sets. Second, existing position prior generation methods in FSS primarily rely on global semantic matching between support and query images, overlooking the intrinsic structural correlations within query images themselves. This limitation becomes particularly problematic when segmenting objects with complex internal structures or in the presence of visually similar distractors, as the initial position priors lack sufficient boundary precision and structural coherence.

To address these challenges, this paper introduces CMaP-SAM, a novel framework that leverages contraction mapping theory~\cite{istratescu2001fixed} to optimize position priors for SAM-driven few-shot segmentation. As shown in Fig.~\ref{fig:motivation} (b),  The proposed CMaP-SAM is built on three key innovations: (1) A contraction mapping module that formulates position prior optimization with mathematical convergence guarantees. This theoretical framework enables iterative refinement of position priors by constructing pixel-wise structural similarity measures that capture intrinsic correlations within query images. (2) An adaptive distribution alignment module that bridges the representation gap between continuous probability priors and SAM's binary mask prompt encoder, preserving the rich information in probability distributions while enabling effective utilization of SAM's segmentation capabilities. (3) A foreground-background decoupled refinement architecture that produces accurate final segmentation masks by effectively handling the inherent asymmetry between foreground objects and background regions, particularly beneficial for objects with complex boundaries. The contribution of this work can be summarized as follows:
\begin{itemize}
  \item CMaP-SAM is proposed to integrate contraction mapping theory with SAM for few-shot segmentation. Position prior optimization is formulated as a Banach contraction mapping with convergence guarantees, enabling iterative refinement of position priors.
  \item An adaptive distribution alignment module is introduced to bridge the representation gap between continuous probability priors and SAM's binary mask prompt encoder.
  \item A foreground-background decoupled refinement architecture is designed to produce accurate segmentation masks by addressing the inherent asymmetry between foreground objects and background regions.
  \item Extensive experiments validate CMaP-SAM's effectiveness, achieving state-of-the-art performance with 71.1 mIoU on PASCAL-$5^i$ and 56.1 mIoU on COCO-$20^i$ datasets.
  \end{itemize}

\section{Related Work}
\subsection{Few-Shot Segmentation}
Existing few-shot segmentation methods can be primarily categorized into prototype-based and matching-based approaches. Prototype-based methods extracted representative features from support images to guide query image segmentation, with early investigations~\cite{rakelly2018conditional,zhang2020sg} employing masked average pooling for singular global prototypes, subsequent research~\cite{yang2020prototype,lang2024few,gairola2020simpropnet,xie2021few,yang2021mining} developing multi-prototype approaches to decompose foreground information, and recent advancements~\cite{fan2022self,mao2022learning,cheng2023hpa} focusing on dynamic prototype generation mechanisms that adapted to cross-image appearance variations. In contrast, matching-based methods established pixel-level correspondences between support and query images, evolving from initial approaches~\cite{min2021hypercorrelation} that computed direct feature similarities to incorporating graph-based mechanisms~\cite{zhang2019pyramid,wang2020few} for structured relationship modeling, with contemporary methods extensively leveraging attention mechanisms—such as~\cite{zhang2021few,shi2022dense} that employed cross-attention for similarity computation and~\cite{peng2023hierarchical} that adopted hierarchical matching strategies for enhanced feature correspondence reliability. It was noteworthy that while prototype-based methods offered computational efficiency but sacrificed spatial details, matching-based approaches preserved structural information at the expense of higher computational complexity.

\subsection{SAM-based Segmentation}

The Segment Anything Model (SAM) \cite{kirillov2023segment} revolutionized zero-shot segmentation through its category-agnostic approach. Subsequent research advanced SAM's capabilities along three strategy: semantic enhancement, domain specialization, and few-shot integration. Semantic enhancement addressed SAM's limited categorical understanding through models like Semantic-SAM~\cite{semanticsam}, which improved granularity via joint training, and OV-SAM~\cite{OVSAM}, which incorporated CLIP-based \cite{radford2021learning} text prompting. Domain specialization tailored SAM to specific fields: RSPrompt~\cite{chen2024rsprompter} improved segmentation quality in remote sensing, MAS-SAM~\cite{yan2024mas} was designed for marine animal segmentation, while SAM-Med2D~\cite{cheng2023sam}, SurgicalSAM~\cite{yue2024surgicalsam} optimized performance for various medical imaging applications through specialized prompt encoders. In addition, Hi-SAM~\cite{ye2024hi} leveraged parameter-efficient fine-tuning for hierarchical text segmentation, and SAM-HQ~\cite{ke2023segment} introduced a learnable high-quality output token to improve performance on high-resolution images. Most recently, few-shot integration transformed SAM from requiring labor-intensive manual prompts to achieving automated segmentation with minimal examples. PerSAM~\cite{zhang2023personalize} enabled single-sample model personalization, Matcher~\cite{liu2023matcher} developed robust bidirectional prompt sampling, and FSS-SAM \cite{feng2024boosting} combined few-shot predictions with SAM processing.

\section{Methodology}

\subsection{Problem Setup}
\label{subsection:problem_setup}
Few-shot segmentation represents a compelling research direction aimed at transferring learned segmentation capabilities to novel object categories with extremely limited supervision. In this learning paradigm, models must acquire generalizable representation knowledge from abundant base-class data ($\mathcal{D}_b$ with categories $\boldsymbol{C}_b$) and subsequently adapt to segment previously unseen novel classes ($\mathcal{D}_n$ with categories $\boldsymbol{C}_n$), where the fundamental constraint $\boldsymbol{C}_b \cap \boldsymbol{C}_n = \emptyset$ maintains evaluation integrity through strict category separation. The prevailing meta-learning approach structures the learning process into episodic tasks $\mathcal{T}$, each comprising a support-query pair $(\mathcal{S}, \mathcal{Q})$. Within the established $N$-way $K$-shot framework, each support set $\mathcal{S}$ contains precisely $N \times K$ labeled image-mask pairs $(I_s, M_s)$, providing the reference examples from which the model must extract discriminative features. Concurrently, the query set $\mathcal{Q}$ consists of $N$ test images $(I_q, M_q)$ awaiting accurate segmentation. The hyperparameters $N$ and $K$ control task complexity and supervision density respectively, with $K$ commonly set to minimal values (1 or 5) to reflect real-world annotation constraints. Throughout this formulation, both support and query images maintain consistent dimensionality ($I_s, I_q \in \mathbb{R}^{H \times W \times 3}$), as do their corresponding segmentation masks ($M_s, M_q \in \mathbb{R}^{H \times W}$).

\subsection{Overview of the Proposed Method: CMaP-SAM}
\begin{figure}[htbp]
    \centering
    \includegraphics[width= 1.0\linewidth]{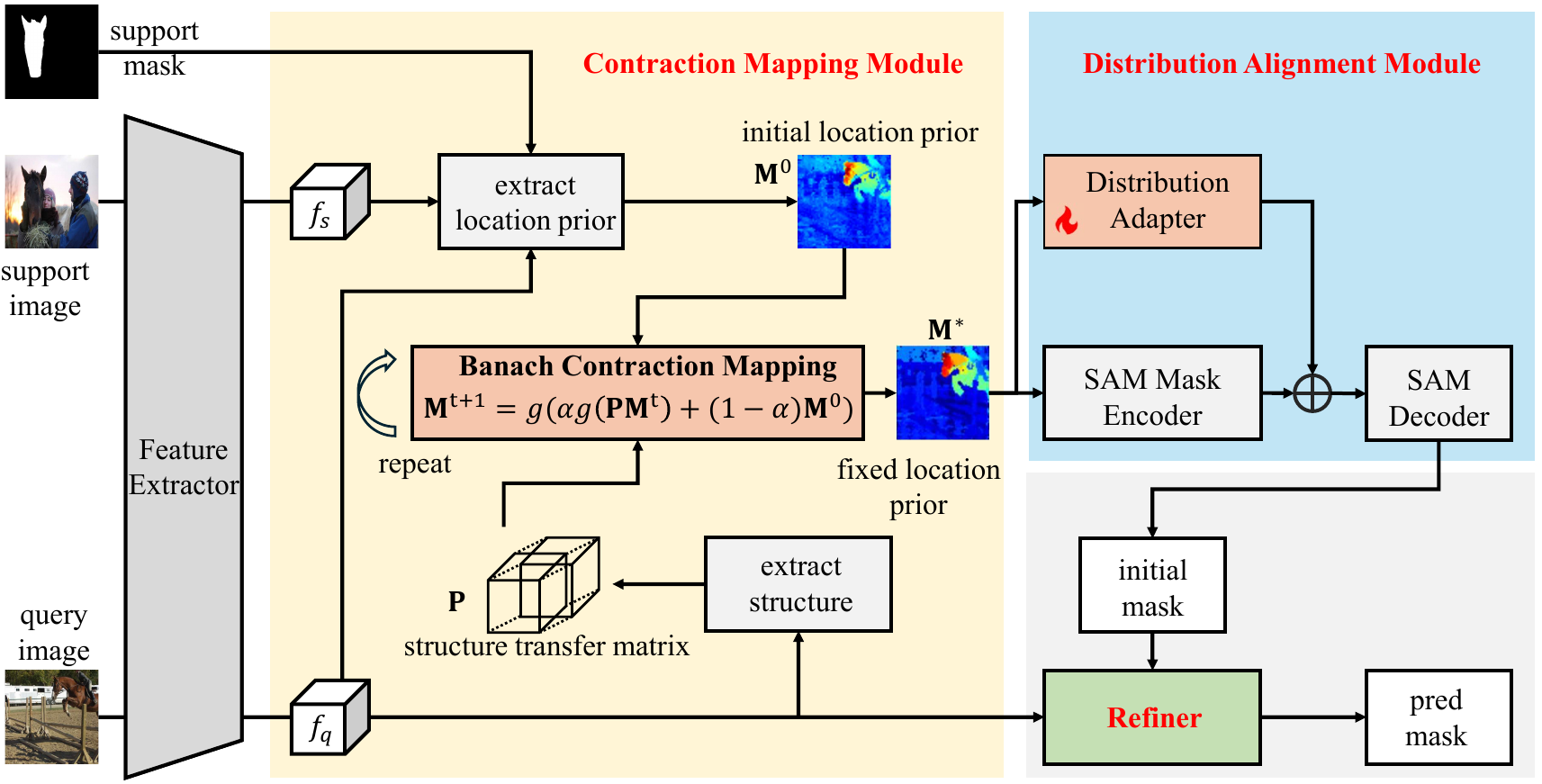}
    \caption{Overview of CMaP-SAM. It consists of three main components: (1) Contraction Mapping Prior Module for position prior construction and optimization; (2) Distribution Alignment Module for bridging continuous priors with SAM's binary mask prompt encoder; and (3) Foreground-Background Decoupled Refinement Module for final segmentation refinement.}
    \label{fig:framework} 
  \end{figure}

As illustrated in Fig.~\ref{fig:framework}, the proposed CMaP-SAM framework comprises three components that work in synergy to enhance few-shot segmentation performance. The first component, the Contraction Mapping Prior Module, establishes and optimizes position priors by leveraging prototype matching between support and query features, followed by an iterative refinement process grounded in contraction mapping theory. This approach effectively captures the intrinsic structural correlations within query images while maintaining valuable semantic guidance from the support set. The second component, the Distribution Alignment Module, addresses a critical challenge by intelligently bridging the gap between the continuous probability distributions generated as priors and SAM's requirement for binary mask inputs. This module preserves the rich uncertainty information embedded within the probability distributions while enabling the effective utilization of SAM's powerful segmentation capabilities. Finally, the Foreground-Background Decoupled Refinement Module enhances segmentation quality through  multi-scale architecture that separately processes foreground and background regions.

\subsection{Contraction Mapping Prior Module}
\label{subsection:CMaP_module}

\subsubsection{Initial Prior Construction}

The position prior~\cite{tian2020prior, shi2022dense} serves as a critical spatial guide in few-shot segmentation, directing attention to potential target regions. Given support images $I_s$ and query image $I_q$, high-dimensional features are first extracted using a pre-trained backbone network $\Phi(\cdot)$:
\begin{equation}
\mathbf{f}_s = \Phi(I_s), \quad \mathbf{f}_q = \Phi(I_q),
\end{equation}
where $\mathbf{f}_s \in \mathbb{R}^{B\times C\times H_s\times W_s}$ and $\mathbf{f}_q \in \mathbb{R}^{B\times C\times H_q\times W_q}$ represent the feature maps of support and query images, respectively. To accurately represent the target class, a class prototype is constructed through mask-guided feature aggregation:
\begin{equation}
\mathbf{proto}_{s} = \frac{\sum(\mathbf{f}_s \odot \mathbf{M}_s)}{\sum\mathbf{M}_s},
\end{equation}
where $\mathbf{proto}{s} \in \mathbb{R}^{B\times C\times 1\times 1}$ represents the class prototype, $\mathbf{M}_s \in \mathbb{R}^{B\times 1\times H_s\times W_s}$ is the support mask, and $\odot$ denotes element-wise multiplication. This mask-guided aggregation effectively isolates target features from background interference. The initial position prior $\mathbf{M}^{0} \in \mathbb{R}^{B\times 1\times H_q\times W_q}$ is computed using cosine similarity between query features and the class prototype:
\begin{equation}
\mathbf{M}^{0}_{raw} = \text{sim}(\mathbf{f}_q, \mathbf{proto}_{s}) = \frac{\langle\mathbf{f}_q, \mathbf{proto}_{s}\rangle}{|\mathbf{f}_q|\cdot|\mathbf{proto}_{s}|},
\end{equation}
where $\text{sim}(\cdot,\cdot)$ denotes cosine similarity, $\langle\cdot,\cdot\rangle$ represents inner product, and $|\cdot|$ represents Euclidean norm. To enhance the discriminative power of the position prior, min-max normalization is applied to map the position prior to the $[0,1]$ range and amplifying the contrast between foreground and background regions:
\begin{equation}
\mathbf{M}^{0} = \frac{\mathbf{M}^{0}_{raw} - \min(\mathbf{M}^{0}_{raw})}{\max(\mathbf{M}^{0}_{raw}) - \min(\mathbf{M}^{0}_{raw})}.
\end{equation}

\subsubsection{Structure Transfer Matrix Construction} Given feature representation $\mathbf{f}_q \in \mathbb{R}^{B\times C\times H_q\times W_q}$ of the query image, its spatial dimensions are first reshaped to sequence form $\mathbf{f}_q \in \mathbb{R}^{B\times C\times N_q}$ where $N_q = H_q \times W_q$, enabling pixel-level correlation modeling. A structure transfer matrix $\mathbf{P} \in \mathbb{R}^{B\times N_q\times N_q}$ is constructed through:

\begin{equation}
\mathbf{P} = \underbrace{\text{softmax}}_{\text{ Normalization}} \left( \underbrace{\mathcal{T}_k}_{\text{Sparsification}} \left( \underbrace{\phi(\mathbf{f}_q) \mathbin{/} t}_{\text{Similarity calculation}} \right) \right), 
\end{equation}
where $\phi(\mathbf{f}_q)$ represents feature similarity measure, implemented as dot product $\mathbf{f}_q^\top \mathbf{f}_q$. This construction involves three key operations: adjusting the raw feature similarity matrix through temperature coefficient $t$; applying sparsification operator $\mathcal{T}_k(\cdot)$ to achieve local selective attention; and applying softmax normalization to obtain a probabilistic transfer matrix. The sparsification operator $\mathcal{T}_k(\cdot)$ is designed as a Top-$k$ selection function:
\begin{equation}
  \mathcal{T}_k(P_{i,j}) =
  \begin{cases}
  P_{i,j}, & \text{if } P_{i,j} \in \text{top-}k(P_i) \\
  -\infty, & \text{otherwise}
  \end{cases}
  \end{equation}
where $P_{i,j}$ represents elements in the similarity matrix, and $\text{top-}k(P_{i})$ denotes the set of $k$ elements with highest similarity values in row $i$. By setting non-Top-$k$ elements to negative infinity, their probability weights after softmax normalization approach zero, achieving strict sparsification. This design reduces computational complexity from $O(N_q^2)$ to $O(kN_q)$, significantly alleviating the computational burden when processing high-resolution feature maps.

\subsubsection{Contraction Mapping Optimization}
The core innovation of the proposed approach is formulating position prior optimization as a Banach contraction mapping~\cite{istratescu2001fixed} with convergence guarantees. This theoretical framework enables iterative refinement of position priors that preserves both semantic guidance from support images and structural correlations in query images. As shown in Fig.~\ref{fig:contraction_mapping}, this iteration comprises three key components: a structural consistency propagation term $g(\mathbf{P} \mathbf{M}^{t})$, an initial semantic anchoring term $(1-\alpha) \mathbf{M}^{0}$, and a piecewise affine normalization function $g(\cdot)$.

\begin{figure}[htbp]
  \centering
  \includegraphics[width=0.9\linewidth]{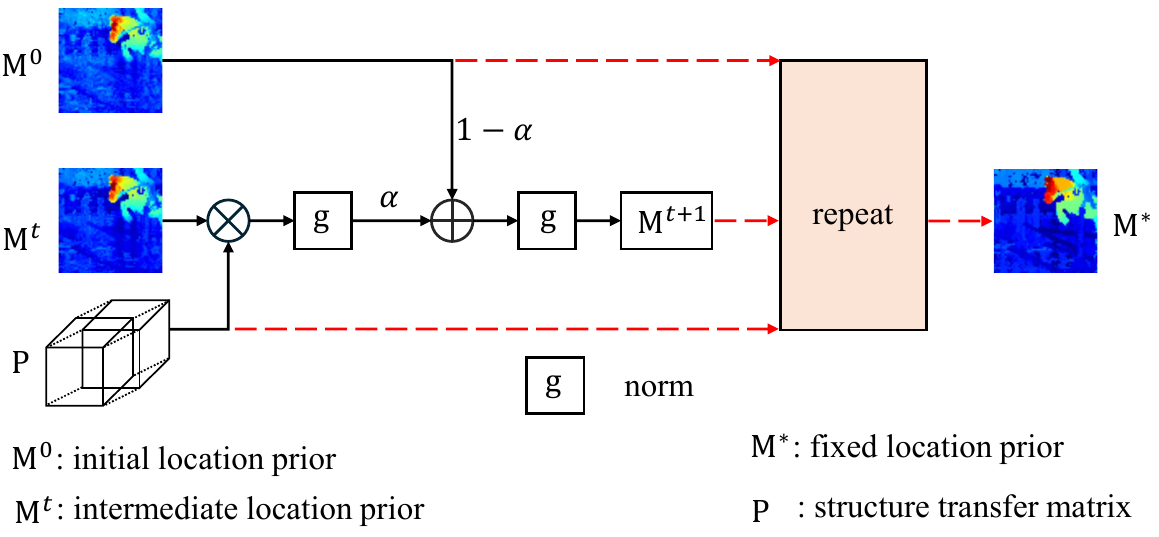}
  \caption{Iterative mapping process for position prior optimization. The process involves three key components: structural consistency propagation, initial semantic anchoring, and piecewise affine normalization.}
  \label{fig:contraction_mapping}
\end{figure}

\begin{equation}
\mathbf{M}^{t+1} = g\Big( \alpha g(\mathbf{P} \mathbf{M}^{t}) + (1-\alpha) \mathbf{M}^{0}\Big).
\label{eq:iterative_mapping}
\end{equation}

The structural consistency propagation term diffuses information based on local structural similarities in the query image, capturing geometric properties and spatial continuity of the target. The initial semantic anchoring term preserves semantic guidance from support images, preventing the iteration process from deviating into local optima based solely on image structures. The normalization function maintains the position prior values within the $[0,1]$ range:

\begin{equation}
g(\mathbf{v}) = \frac{\mathbf{v} - \min(\mathbf{v})}{\max\left(\max(\mathbf{v}) - \min(\mathbf{v}), \delta\right) + \varepsilon},
\label{eq:normalization_new}
\end{equation}
where $\delta > 0$ is a dynamic range lower bound and $\varepsilon = 10^{-8}$ ensures numerical stability. Unlike standard min-max normalization, this function introduces a dynamic range constraint, preventing excessive noise amplification while preserving distribution morphology. The balancing weight $\alpha \in (0,1]$ controls the trade-off between structural information and semantic guidance. During iteration, the position prior gradually converges to a fixed prior $\mathbf{M}^*$:

\begin{equation}
\mathbf{M}^* = g\Big(\alpha g(\mathbf{P} \mathbf{M}^*) + (1-\alpha) \mathbf{M}^{0}\Big).
\end{equation}

\textbf{Convergence Analysis:} Theoretical guarantees for the convergence of the iterative optimization process are provided:

\begin{theorem}
Given the iterative mapping defined in Equation~\ref{eq:iterative_mapping} with $\alpha \in (0,1]$ and $\delta > 0$ such that $\frac{\alpha}{(\delta + \varepsilon)^2} < 1$, the sequence ${\mathbf{M}^t}$ converges to a fixed prior $\mathbf{M}^*$ in the complete space $\bigl([0,1]^{B \times N_q}, |\cdot|_\infty\bigr)$.
\end{theorem}

\begin{proof}
The theorem is established through the following steps:

\textbf{Row-stochastic property}: The transfer matrix $\mathbf{P}$ satisfies row-stochasticity $\sum_j \mathbf{P}{b,i,j} = 1$, ensuring for any $\mathbf{v} \in \mathbb{R}^{B \times N_q}$:
\begin{equation}
|\mathbf{P}\mathbf{v}|\infty \leq |\mathbf{v}|_\infty.
\end{equation}

\textbf{Lipschitz continuity of normalization}: Let $\Delta_b = \max\bigl(\max(\mathbf{v}_b) - \min(\mathbf{v}_b), \delta\bigr)$, then the normalization mapping satisfies:
\begin{equation}
|g(\mathbf{v}_1) - g(\mathbf{v}2)|\infty \leq \frac{1}{\delta + \varepsilon} |\mathbf{v}_1 - \mathbf{v}2|\infty.
\end{equation}

\textbf{Dynamic range lower bound}: During iteration, the dynamic range of the combination term $\alpha, g(\mathbf{P}\mathbf{M}^t) + (1-\alpha)\mathbf{M}^0$ is lower bounded by:
\begin{equation}
\Delta_b \geq (1-\alpha)\delta,
\end{equation}
since $\mathbf{M}^0 \in [0,1]$ and $g(\mathbf{P}\mathbf{M}^t) \in [0,1]$.

\textbf{Contraction property of iteration mapping}: Defining the complete iteration mapping as $f(\mathbf{M}) = g\Bigl( \alpha, g(\mathbf{P}\mathbf{M}) + (1-\alpha)\mathbf{M}^0 \Bigr)$, its Lipschitz constant is:
\begin{equation}
L = \frac{\alpha}{(\delta + \varepsilon)^2},
\end{equation}
with parameter settings $\delta = 0.2$ and $\alpha = 0.03$, $L < 1$ is satisfied, meeting the contraction mapping condition.

\textbf{Convergence conclusion}: According to the Banach fixed-point theorem~\cite{istratescu2001fixed}, the iteration sequence ${\mathbf{M}^t}$ converges to a fixed prior $\mathbf{M}^*$ in the complete space $\bigl([0,1]^{B \times N_q}, |\cdot|_\infty\bigr)$.
\end{proof}

\subsection{Mask Alignment Module}

While the contraction mapping prior module generates high-quality continuous position priors, SAM's mask prompt encoder expects binary mask inputs. This representation gap poses a significant challenge for effective integration. Existing methods~\cite{feng2024boosting} typically adopt a ``generate coarse mask, extract discrete points'' strategy, which results in substantial information loss during quantization.

\begin{figure}[htbp]
    \centering
    \includegraphics[width= 0.7\linewidth]{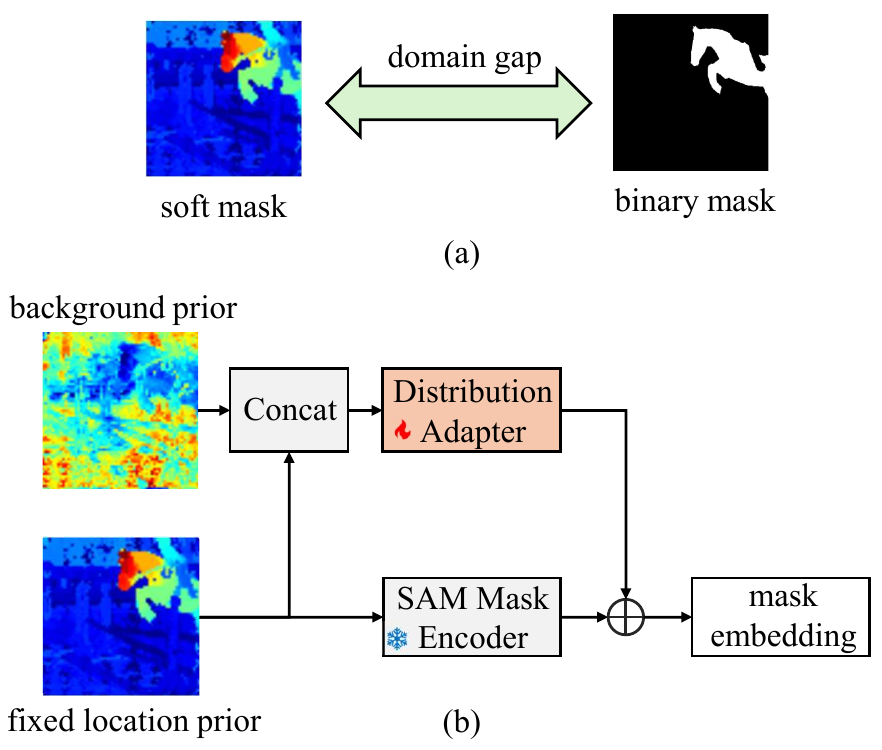}
    \caption{Mask alignment module. The dual-prior mask alignment module bridges the representation gap between continuous position priors and SAM's binary mask requirements, preserving rich information in probability distributions.}
    \label{fig:mask_alignment}
\end{figure}

To address this challenge, a dual-prior mask alignment module is proposed that bridges the representation gap while preserving the rich information in probability distributions. As illustrated in Fig.~\ref{fig:mask_alignment}, the foreground prior $\mathbf{M}^{*} \in [0,1]^{H \times W}$ and background prior $\mathbf{M}_{\text{bg}} \in [0,1]^{H \times W}$ are first concatenated to construct an enhanced dual-prior representation:

\begin{equation}
\mathbf{M}_{\text{dual}} = \text{Concat}[\mathbf{M}^{*}, \mathbf{M}_{\text{bg}}] \in \mathbb{R}^{2 \times H \times W},
\end{equation}
this design fully considers the complementary relationship between foreground and background information. Next, a lightweight adapter network is introduced to process the enhanced dual-prior representation:

\begin{equation}
\text{emb}_{\text{adapt}} = \text{adapter}(\mathbf{M}_{\text{dual}}),
\end{equation}
where the adapter network consists of a $3 \times 3$ convolution layer for local feature extraction, a ReLU layer, and a $1 \times 1$ convolution layer for feature dimension matching. To achieve dynamic fusion of pretrained mask encoder features and adapted features, a learnable fusion coefficient $\lambda$ is introduced:

\begin{equation}
\text{emb}_{\text{mask}} = \mathcal{E}(\mathbf{M}^{*}_{\text{bin}}) + \lambda \cdot \text{emb}_{\text{adapt}},
\end{equation}
where $\mathcal{E}(\cdot)$ represents the embedding extraction operation of the SAM mask encoder, $\mathbf{M}^{}_{\text{bin}} = \mathbf{M}^{} > 0.5$ is the binarized version of the foreground prior, and $\lambda$ is a learnable scalar parameter that dynamically adjusts the contribution ratio of pretrained features and adapted features.

\subsection{Foreground-Background Decoupled Refinement Module}

In few-shot segmentation, foreground objects exhibit coherent semantic structures with consistent feature distributions, while background regions display complex variations and contextual associations. This fundamental asymmetry necessitates treating foreground and background regions differently during refinement rather than as a homogeneous entity. Building on this observation, we propose a Foreground-Background Decoupled Refinement Module that integrates two complementary prior masks: $\mathbf{M}^* \in \mathbb{R}^{H \times W}$ from converged position priors and $\mathbf{M}_{\text{sam}} \in \mathbb{R}^{H \times W}$ from the SAM decoder. While $\mathbf{M}^*$ identifies discriminative semantic regions with high reliability but lower boundary precision, $\mathbf{M}_{\text{sam}}$ provides detailed boundary information through SAM's strong prompt segmentation capabilities.

To utilize multi-scale feature representation, a four-level feature pyramid network with scale parameters $s \in \{512, 256, 128, 64\}$ is constructed. At each scale level, a feature decoupling network separately extracts discriminative features for foreground and background:

\begin{align}
  \mathbf{F}_{\text{loc\_fg}}^s &= \mathcal{D}(\mathbf{F}^s \odot \mathbf{M}^{*s}), \\
  \mathbf{F}_{\text{loc\_bg}}^s &= \mathcal{D}(\mathbf{F}^s \odot (1 - \mathbf{M}^{*s})), \\
  \mathbf{F}_{\text{sam\_fg}}^s &= \mathcal{D}(\mathbf{F}^s \odot \mathbf{M}_{\text{sam}}^s), \\
  \mathbf{F}_{\text{sam\_bg}}^s &= \mathcal{D}(\mathbf{F}^s \odot (1 - \mathbf{M}_{\text{sam}}^s)),
  \end{align}
where $\mathbf{F}^s \in \mathbb{R}^{C \times H_s \times W_s}$ is the input feature map and $\mathcal{D}(\cdot)$ represents the feature decoupling network. Considering that accurate boundary localization is crucial for segmentation quality, a dedicated edge-aware module is introduced:

\begin{equation}
\mathbf{E}^s = \mathcal{B}(\mathbf{F}^s, |\mathbf{M}^{*s} - \mathbf{M}_{\text{sam}}^s|),
\end{equation}
where $\mathcal{B}(\cdot)$ is an edge-aware module that generates boundary attention weights through the mask difference map $|\mathbf{M}^{*s} - \mathbf{M}_{\text{sam}}^s|$. The multi-source feature fusion process can be formalized as:

\begin{equation}
  \mathbf{F}_{\text{ms}}^s = \mathcal{M}\left(\mathbf{F}_{\text{loc\_fg}}^s \Vert \mathbf{F}_{\text{loc\_bg}}^s \Vert \mathbf{F}_{\text{sam\_fg}}^s \Vert \mathbf{F}_{\text{sam\_bg}}^s \Vert \mathbf{E}^s \right),
  \end{equation}
where $\Vert$ represents channel-dimension concatenation, and $\mathcal{M}(\cdot)$ is a multi-modal fusion function composed of gated attention units with weight coefficients $\alpha_i^s \in [0,1]$ calculated as:

\begin{equation}
\alpha_i^s = \sigma\left(\mathbf{W}_g^s \cdot \mathcal{G}(\mathbf{F}_i^s) + \mathbf{b}_g^s\right),
\end{equation}
where $\sigma(\cdot)$ is the Sigmoid activation function and $\mathcal{G}(\cdot)$ is a global context extraction operator. This mechanism enhances effective features by measuring the correlation between feature channels and target regions. The progressive optimization process implements cross-scale feature refinement through a cascade architecture:
\begin{equation}
\mathbf{\hat{M}}^s = \Psi^s\left(\mathcal{C}(\mathbf{F}_{\text{ms}}^s, \mathcal{U}(\mathbf{\hat{M}}^{s/2})) \right),
\end{equation}
where $\mathcal{C}(\cdot,\cdot)$ represents cross-scale feature calibration operations including spatially-aware feature alignment, $\mathcal{U}(\cdot)$ is an upsampling operation, and $\Psi^s(\cdot)$ is an optimization unit consisting of three cascaded depth-separable convolutions, enhancing local-global feature interaction while maintaining computational efficiency.

The final segmentation result is generated through multi-scale feature aggregation:
\begin{equation}
\mathbf{M}_{\text{refined}} = \sum_{s \in \mathcal{S}} \mathcal{R}_s(\mathbf{\hat{M}}^s) \odot \mathbf{A}^s,
\end{equation}
where $\mathcal{S} = \{64,128,256,512\}$ is the scale set, $\mathcal{R}_s(\cdot)$ is a scale-adaptive convolution layer, and $\mathbf{A}^s \in \mathbb{R}^{H \times W}$ is a scale weight matrix learned through an attention mechanism. This design achieves optimal balance between semantic consistency and boundary precision by explicitly modeling spatial correlations of multi-scale features.

\subsubsection{Total Loss}
The total loss function is defined as a weighted combination of the initial segmentation loss and the refinement loss, with respective weights of 0.3 and 0.7. At each stage, the loss comprises two components: Binary Cross-Entropy loss and Dice loss. The Dice loss is formulated to optimize the spatial continuity of the segmentation results by maximizing the overlap between the predicted and ground truth regions, while the BCE loss focuses on improving the classification accuracy of the segmentation outputs by penalizing incorrect predictions. Mathematically, the Dice loss is expressed as:

\begin{equation}
\mathcal{L}_{\text{dice}} = 1 - \frac{2 \sum y_{i,j} \cdot \hat{y}_{i,j}}{\sum y_{i,j}^2 + \sum \hat{y}_{i,j}^2},
\end{equation}
and the BCE loss is given by:

\begin{equation}
\mathcal{L}_{\text{bce}} = -\frac{1}{H \times W} \sum \big(y_{i,j}\log(\hat{y}_{i,j}) + (1-y_{i,j})\log(1-\hat{y}_{i,j})\big),
\end{equation}
where \( y_{i,j} \) represents the ground truth and \( \hat{y}_{i,j} \) denotes the predicted value at pixel \((i, j)\).

\section{Experiments}

\subsection{Datasets and Evaluation Metrics}

\textbf{Datasets}: Following established protocols~\cite{tian2020prior,shi2022dense,min2021hypercorrelation,wang2019panet,MASNet, bor}, three benchmarks were utilized for evaluation: PASCAL-5$^i$\cite{shaban2017one}, COCO-20$^i$\cite{nguyen2019feature}, and FSS-1000~\cite{li2020fss}. PASCAL-5$^i$ is constructed from PASCAL VOC 2012~\cite{everingham2010pascal} with SBD~\cite{hariharan2011semantic} annotations, containing 20 semantic categories divided into 4 folds of 5 classes each, with 10,582 training and 1,449 validation images. COCO-20$^i$ is built upon MS COCO~\cite{lin2014microsoft}, with 80 categories in 4 folds of 20 classes each (82,081 training, 40,137 validation images), presenting challenges through dense multi-instance annotations, complex scenes with occlusions, and numerous small objects. FSS-1000~\cite{li2020fss} focuses on fine-grained segmentation across 1,000 categories with 10,000 samples (split 520:240:240), characterized by high inter-class similarity and limited samples per category.

\textbf{Evaluation Metrics}: Performance assessment was carried out using cross-fold validation, where models were trained on three folds and tested on the remaining fold to evaluate their generalization capability to novel classes. Two complementary metrics were employed to ensure a comprehensive evaluation. The first metric, class-wise mean Intersection-over-Union ($\text{mIoU}$), calculated the average IoU across all test classes, defined as $\text{mIoU} = \frac{1}{C} \sum_{c=1}^{C} \text{IoU}_c$. The second metric, Foreground-Background IoU ($\text{FB-IoU}$), measured the average IoU between foreground and background regions, defined as $\text{FB-IoU} = \frac{1}{2} (\text{IoU}_\text{F} + \text{IoU}_\text{B})$.

\subsection{Implementation Details}

\textbf{Implementation Details}: All experiments were implemented on the PyTorch framework~\footnote{https://pytorch.org} using 4 NVIDIA RTX 3090 GPUs. Following previous method~\cite{feng2024boosting}, DINOv2 (ViT-Base-P14)~\cite{oquab2023dinov2} was employed as the feature extractor, with input images resized to 896$\times$896 pixels to extract 64$\times$64 feature maps. The SAM model operated at 1024$\times$1024 resolution, producing matching 64$\times$64 feature maps to maintain spatial consistency. Position prior generation was configured with 8 neighbors and a temperature coefficient of 0.1. Training was conducted using AdamW optimizer with a 1e-4 learning rate and batch size of 2 for 20 epochs on PASCAL-5$^i$ and FSS-1000, and 5 epochs on COCO-20$^i$. Each training process on PASCAL-5$^i$ was completed in 0.7 GPU day, with the peak GPU memory utilization reaching 22GB per GPU during training. Data augmentation included random horizontal flipping, scale transformations (0.8-1.2$\times$), rotation ($\pm$10°), and brightness/contrast adjustments. For K-shot evaluations, predictions were averaged as $\mathbf{M}_{\text{final}} = \frac{1}{K} \sum_{k=1}^K \mathbf{M}_k$.

 \subsection{Comparison with State-of-the-Art Methods}

 \begin{table}[htbp]
  \centering
  \caption{Comparison of 1-shot performance on $\text{PASCAL-}5^i$ dataset. Best in \textbf{bold}, second best in \underline{underlined}.}
  \label{tab:CMaP_pascal_1shot}
      \scalebox{0.7}{
      \begin{tabular}{cccc|cccccc}
        \toprule
        \multirow{2}{*}{\shortstack{\textbf{Category}}} & \multirow{2}{*}{\shortstack{\textbf{Backbone}}} & \multirow{2}{*}{\textbf{Method}} & \multirow{2}{*}{\textbf{Publication}} & \multicolumn{6}{c}{\textbf{1-shot Performance}} \\ 
        & & && $\mathbf{5^{0}}$ & $\mathbf{5^{1}}$ &$\mathbf{5^{2}}$ & $\mathbf{5^{3}}$ & \textbf{mIoU} & \textbf{FB-IoU} \\
        \midrule
        \multirow{12}{*}{\shortstack{Specialized \\ Models}}& \multirow{7}{*}{ResNet-50}  
        & PFENet~\cite{tian2020prior} &TPAMI'20 & 61.7 & 69.5 & 55.4 & 56.3 & 60.8 & 73.3  \\
        && RePRI~\cite{boudiaf2021few}& CVPR'21 & 59.8 & 68.3 & 62.1 & 48.5 & 59.7 & -  \\
        && HSNet~\cite{min2021hypercorrelation} &ICCV'21 & 64.3 & 70.7 & 60.3 & 60.5 & 64.0 & 76.7\\
        && DCAMA~\cite{shi2022dense}&ECCV'22  & 67.5 & 72.3 & 59.6 & 59.0 & 64.6 & 75.7 \\

        && ABCNet~\cite{Wang_2023_CVPR} &CVPR'23 &68.8&73.4&62.3&59.5&66.0&76.0  \\ 
        && MASNet~\cite{MASNet}&NEUCOM'24 & 61.7&68.2&67.4&50.3&61.9&-\\
        && BOR~\cite{bor} &NEUCOM'25& 69.2 & \underline{74.7} & \underline{67.8} & 60.4 & 68.0 & - \\

          \cmidrule(lr){2-9}
         & \multirow{5}{*}{ResNet-101} 
         & PFENet~\cite{tian2020prior} &TPAMI'20 & 60.5 & 69.4 & 54.4 & 55.9 & 60.1 & 72.9  \\
         && HSNet~\cite{min2021hypercorrelation}&ICCV'21  & 67.3 & 72.3 & 62.0 & 63.1 & 66.2 & 77.6   \\
         && DCAMA~\cite{shi2022dense} &ECCV'22 & 65.4 & 71.4 & 63.2 & 58.3 & 64.6 & 77.6  \\ 
         && HPA~\cite{cheng2023hpa}&TPAMI'23   & 66.4 & 72.7 & 64.1 & 59.4 & 65.6 & 76.6  \\
        && ABCNet~\cite{Wang_2023_CVPR} &CVPR'23 &65.3&72.9&65.0&59.3&65.6&\underline{78.5}   \\  

        \midrule

      \multirow{4}{*}{\shortstack{General \\ Models}} & \multirow{4}{*}{SAM} 
      & FSS-SAM~\cite{feng2024boosting} & Arxiv'24& \underline{70.2} & 74.0 & \textbf{67.9} & 62.0 & \underline{68.5} & - \\

      && PerSAM~\cite{zhang2023personalize}&ICLR'24 & - & - & - &- & 48.5 & - \\

      && Matcher~\cite{liu2023matcher}&ICLR'24 & 67.7 & 70.7 & 66.9 & \textbf{67.0} & 68.1 & - \\

      &&  CMaP-SAM (ours) &- & \textbf{78.1 }& \textbf{75.5} & 65.0 & \underline{66.0} & \textbf{71.1} &\textbf{80.8}  \\
        \bottomrule
      \end{tabular}
      }
\end{table}

\begin{table}[htbp]
  \centering
  \caption{Comparison of 5-shot performance on $\text{PASCAL-}5^i$ dataset. Best in \textbf{bold}, second best in \underline{underlined}.}
  \label{tab:CMaP_pascal_5shot}
  \scalebox{0.7}{
      \begin{tabular}{cccc|cccccc}
        \toprule
        \multirow{2}{*}{\shortstack{\textbf{Category}}} & \multirow{2}{*}{\shortstack{\textbf{Backbone}}} & \multirow{2}{*}{\textbf{Method}} & \multirow{2}{*}{\textbf{Publication}} & \multicolumn{6}{c}{\textbf{5-shot Performance}} \\ 
        && & & $\mathbf{5^{0}}$ & $\mathbf{5^{1}}$ &$\mathbf{5^{2}}$ & $\mathbf{5^{3}}$ & \textbf{mIoU} & \textbf{FB-IoU} \\
        \midrule
        \multirow{12}{*}{\shortstack{Specialized \\ Models}}& \multirow{7}{*}{ResNet-50}  
        & PFENet~\cite{tian2020prior} &TPAMI'20 & 63.1 & 70.7 & 55.8 & 57.9 & 61.9 & 73.9  \\
        && RePRI~\cite{boudiaf2021few} &CVPR'21 & 64.6 & 71.4 & \underline{71.1} & 59.3 & 66.6 & -   \\
        && HSNet~\cite{min2021hypercorrelation} &ICCV'21 & 70.3 & 73.2 & 67.4 & 67.1 & 69.5 & 80.6\\
        && DCAMA~\cite{shi2022dense} &ECCV'22 &  70.5 & 73.9 & 63.7 & 65.8 & 68.5 & 79.5 \\
        && ABCNet~\cite{Wang_2023_CVPR} &CVPR'23 &71.7&74.2&65.4&67.0&69.6&80.0 \\  
        && MASNet~\cite{MASNet}& NEUCOM'24 & 67.7&72.4&73.6&59.9&68.4&-\\
        && BOR~\cite{bor} &NEUCOM'25& 70.8 & 75.9 & \underline{71.1} & 67.9 & 71.4 & - \\
          \cmidrule(lr){2-9}

         & \multirow{5}{*}{ResNet-101} 
         & PFENet~\cite{tian2020prior}&TPAMI'20  & 62.8 & 70.4 & 54.9 & 57.6 & 61.4 & 73.5 \\
         && RePRI~\cite{boudiaf2021few}&CVPR'21 & 66.2 & 71.4 & 67.0 & 57.7 & 65.6 & - \\
         && HSNet~\cite{min2021hypercorrelation}&ICCV'21  & 71.8 & 74.4 & 67.0 & 68.3 & 70.4 & 80.6   \\
         && DCAMA~\cite{shi2022dense} &ECCV'22 & 70.7 & 73.7 & 66.8 & 61.9 & 68.3 & \underline{80.8}  \\ 
        && ABCNet~\cite{Wang_2023_CVPR}&CVPR'23 &71.4&75.0&68.2&63.1&69.4&\underline{80.8}  \\  

        \midrule

        \multirow{3}{*}{\shortstack{General \\ Models}} & \multirow{3}{*}{SAM} 
        & FSS-SAM~\cite{feng2024boosting}&Arxiv'24 &  \underline{71.9} & 75.3 & \underline{71.1} & 68.3 & 71.6 & - \\

        && Matcher~\cite{liu2023matcher}&ICLR'24 & 71.4 & \underline{77.5} & \textbf{74.1} & \underline{72.8} & \underline{74.0} & -  \\

        &&  CMaP-SAM (ours)  &- & \textbf{82.2}  & \textbf{81.2} & 68.5 & \textbf{73.2} & \textbf{76.3} & \textbf{85.1}\\
        \bottomrule
      \end{tabular}
  }
\end{table}

Table~\ref{tab:CMaP_pascal_1shot} and Table~\ref{tab:CMaP_pascal_5shot} show the 1-shot and 5-shot performance on the $\text{PASCAL-}5^i$ dataset, respectively. The proposed CMaP-SAM method consistently demonstrated superior performance compared to existing state-of-the-art methods across various backbones. In the 1-shot setting, CMaP-SAM achieved an mIoU of 71.1 and an FB-IoU of 80.8, exceeding the previous best specialized method (ABCNet) by 5.5 points in mIoU. Additionally, CMaP-SAM outperformed the best general model (FSS-SAM) by 2.6 points in mIoU. In the 5-shot setting, CMaP-SAM further improved its performance, achieving an mIoU of 76.3 and an FB-IoU of 85.1.

\begin{table}[t]
  \centering
  \caption{Comparison of 1-shot performance on $\text{COCO-}20^i$ dataset. Best in \textbf{bold}, second best in \underline{underlined}.}
  \label{tab:CMaP_coco_1shot}
      \scalebox{0.7}{
      \begin{tabular}{cccc|cccccc}
        \toprule
        \multirow{2}{*}{\shortstack{\textbf{Category}}} & \multirow{2}{*}{\shortstack{\textbf{Backbone}}} & \multirow{2}{*}{\textbf{Method}} & \multirow{2}{*}{\textbf{Publication}} & \multicolumn{6}{c}{\textbf{1-shot Performance}} \\ 
        & & && $\mathbf{20^{0}}$ & $\mathbf{20^{1}}$ &$\mathbf{20^{2}}$ & $\mathbf{20^{3}}$ & \textbf{mIoU} & \textbf{FB-IoU} \\
        \midrule
        \multirow{12}{*}{\shortstack{Specialized \\ Models}}& \multirow{6}{*}{ResNet-50}  
        & PFENet~\cite{tian2020prior} &TPAMI'20 & 36.5 & 38.6 & {34.5} & {33.8} & {35.8} & -  \\
        && RePRI~\cite{boudiaf2021few} &CVPR'21& 32.0 & 38.7 &  32.7 & 33.1 & 34.1 & -  \\
        && HSNet~\cite{min2021hypercorrelation} &ICCV'21 & 36.3 & 43.1 & 38.7 & 38.7 & 39.2 & 68.2\\
        && DCAMA~\cite{shi2022dense} &ECCV'22 &  41.9 & 45.1 & 44.4 & 41.7 & 43.3 & 69.5\\
        && ABCNet~\cite{Wang_2023_CVPR} &CVPR'23 &42.3&46.2&46.0&42.0&44.1&\\
        && BOR~\cite{bor} &NEUCOM'25 & 43.7 & 53.1 & 50.8 & 46.0 & 48.4 & - \\ 

          \cmidrule(lr){2-9}
         & \multirow{5}{*}{ResNet-101} 
         & PFENet~\cite{tian2020prior} &TPAMI'20 & {36.8} & {41.8} & {38.7} & {36.7} & {38.5} & {63.0}  \\
         && HSNet~\cite{min2021hypercorrelation}&ICCV'21  &  {37.2} & {44.1} & {42.4} & {41.3} & {41.2} & 69.1    \\
         && DCAMA~\cite{shi2022dense} &ECCV'22 &  41.5 & 46.2 & 45.2 & 41.3 & 43.5 & \underline{69.9}  \\ 
         && SSP~\cite{fan2022self} &ECCV'22  & 39.1 & 45.1& 42.7 & 41.2& 42.0& -  \\
         && HPA~\cite{cheng2023hpa} &TPAMI'23  &43.1 & 50.0 & 44.8 & 45.2 & 45.8 & 68.4 \\
         && MASNet~\cite{MASNet} &NEUCOM'24 & 50.6&39.5&33.1&30.0&38.3&-\\
         
        \midrule

          \multirow{5}{*}{\shortstack{General \\ Models}} & \multirow{5}{*}{SAM} 
          & FSS-SAM~\cite{feng2024boosting} &Arxiv'24 &  39.1 & 50.4 & 48.4 & 43.1 & 45.3 & - \\

          && Matcher~\cite{liu2023matcher}&ICLR'24 & \underline{52.7} & 53.5 & 52.6& \underline{52.1} & 52.7 & - \\

          && PerSAM~\cite{zhang2023personalize}&ICLR'24 & - & - & - &- & 23.5 & - \\

          && VRP-SAM~\cite{sun2024vrp}&CVPR'24 & 48.1 & \underline{55.8} & \textbf{60.0} & 51.6  &\underline{53.9}& - \\

        &&  CMaP-SAM (ours)  &- &  \textbf{54.0} &\textbf{56.1}& \underline{58.9}& \textbf{55.3}& \textbf{56.1} & \textbf{75.8}\\
        \bottomrule
      \end{tabular}}
\end{table}

Table~\ref{tab:CMaP_coco_1shot} and Table~\ref{tab:CMaP_coco_5shot} provide a performance comparison of the proposed CMaP method with existing approaches on the COCO-$20^i$ dataset for 1-shot and 5-shot tasks, respectively. In the general model category, CMaP-SAM demonstrated significant performance improvements over FSS-SAM~\cite{feng2024boosting}, achieving an average mIoU of 56.1 and an FB-IoU of 75.8 in the 1-shot setting. In the 5-shot setting, CMaP-SAM achieved an average mIoU of 65.3, outperforming the best general model (FSS-SAM) by 16.1 points.

\begin{table}[t]
  \centering
  \caption{Comparison of 5-shot performance on $\text{COCO-}20^i$ dataset. Best in \textbf{bold}, second best in \underline{underlined}.}
  \label{tab:CMaP_coco_5shot}
      \scalebox{0.7}{
      \begin{tabular}{cccc|cccccc}
        \toprule
        \multirow{2}{*}{\shortstack{\textbf{Category}}} & \multirow{2}{*}{\shortstack{\textbf{Backbone}}} & \multirow{2}{*}{\textbf{Method}} & \multirow{2}{*}{\textbf{Publication}} & \multicolumn{6}{c}{\textbf{5-shot Performance}} \\ 
        & & && $\mathbf{20^{0}}$ & $\mathbf{20^{1}}$ &$\mathbf{20^{2}}$ & $\mathbf{20^{3}}$ & \textbf{mIoU} & \textbf{FB-IoU} \\
        \midrule
        \multirow{12}{*}{\shortstack{Specialized \\ Models}}& \multirow{6}{*}{ResNet-50}  
        & PFENet~\cite{tian2020prior}&TPAMI'20  & 36.5 & 43.3 & 37.8 & 38.4 & 39.0  & -  \\
        && RePRI~\cite{boudiaf2021few} &CVPR'21 &39.3 & 45.4 & 39.7 & 41.8 & 41.6 & -   \\
        && HSNet~\cite{min2021hypercorrelation} &ICCV'21  &43.3 & 51.3 & 48.2 & 45.0 & 46.9 &70.7\\
        && DCAMA~\cite{shi2022dense} &ECCV'22 &  45.9 & 50.5 & 50.7 & 46.0 & 48.3 & 71.7  \\
        && ABCNet~\cite{Wang_2023_CVPR} & CVPR'23  &45.5&51.7&52.6&46.4&49.1&72.7\\  
        && BOR~\cite{bor} & NEUCOM'25 & 52.3 & 57.9 & 53.3 & 49.7 & 53.3 & - \\

          \cmidrule(lr){2-9}
         & \multirow{6}{*}{ResNet-101} 
         & PFENet~\cite{tian2020prior} & TPAMI'20  & {40.4} & {46.8} & {43.2} & {40.5} & {42.7} & {65.8}  \\
         && HSNet~\cite{min2021hypercorrelation} & ICCV'21  & {45.9} & {53.0} & {51.8} & {47.1} & {49.5} & {72.4}  \\
         && DCAMA~\cite{shi2022dense} & ECCV'22 &  48.0 & \underline{58.0} & \underline{54.3} & 47.1 & 51.9 & 73.3  \\ 
         && SSP~\cite{fan2022self} & ECCV'22  & 47.4 & 54.5 & 50.4 & 49.6 & 50.2 & -  \\
         && HPA~\cite{cheng2023hpa} & TPAMI'23  & \underline{49.2} & 57.8 & 52.0 & \underline{50.6} & \underline{52.4} & \underline{74.0}  \\
         
         && MASNet~\cite{MASNet} &NEUCOM'24 & 56.7&47.0&38.7&40.4&45.7&-\\

        \midrule
      \multirow{2}{*}{\shortstack{General \\ Models}} & \multirow{2}{*}{SAM} 
      & FSS-SAM~\cite{feng2024boosting} & Arxiv'24 &  47.3 & 54.1 & 48.2 & 47.1 & 49.2 & - \\
        &&  CMaP-SAM & - & \textbf{61.3} & \textbf{66.8} & \textbf{67.3} & \textbf{65.9} & \textbf{65.3} & \textbf{81.1} \\
        \bottomrule
      \end{tabular}}
\end{table}

Table~\ref{tab:comparison_PGMA-Net_GF-SAM} provides a detailed comparison of the proposed CMaP-SAM method with PGMA-Net~\cite{chen2024visual} and GF-SAM~\cite{zhang2024bridge} on both $\text{PASCAL-}5^i$ and $\text{COCO-}20^i$ datasets. The results indicate that CMaP-SAM outperforms PGMA-Net in both datasets, achieving an mIoU of 75.6 on $\text{PASCAL-}5^i$ and 60.6 on $\text{COCO-}20^i$ when using CLIP, while maintaining competitive performance without CLIP. Additionally, CMaP-SAM+ (using SAM-Huge) surpasses GF-SAM by 2.1 points on $\text{PASCAL-}5^i$ and 1.8 points on $\text{COCO-}20^i$, demonstrating the effectiveness of the proposed method in enhancing segmentation performance.

\begin{table}[t]
  \centering
  \caption{In-depth comparison with PGMA-Net~\cite{chen2024visual} and GF-SAM~\cite{zhang2024bridge} on $\text{PASCAL-}5^i$  and $\text{COCO-}20^i$ datasets. Best in \textbf{bold}, second best in \underline{underlined}.}
  \label{tab:comparison_PGMA-Net_GF-SAM}
  \scalebox{0.75}{
  \begin{tabular}{cc|c|c}
      \toprule
      \multirow{2}{*}{\textbf{Method}} & \multirow{2}{*}{\textbf{Details}}  & \textbf{$\text{PASCAL-}5^i$} & \textbf{$\text{COCO-}20^i$} \\
      &  &\textbf{mIoU} & \textbf{mIoU} \\
      \midrule
      PGMA-Net~\cite{chen2024visual} & using CLIP  & 
      \underline{74.1} & \underline{59.4} \\
      CMaP-SAM (Ours)  & no CLIP & 71.1 & 56.1 \\
      CMaP-SAM* (Ours) & using CLIP & \textbf{75.6} & \textbf{60.6} \\
      \midrule
      GF-SAM~\cite{zhang2024bridge} & SAM-Huge   & \underline{72.1} & \underline{58.7} \\
      CMaP-SAM (Ours) & SAM-Base   & 71.1 & 56.1 \\
      CMaP-SAM+ (Ours) & SAM-Huge  & \textbf{74.2} & \textbf{60.5} \\
      \bottomrule
  \end{tabular}
  }
\end{table}

Table~\ref{tab:CMaP_fss1000_comparison} provides a performance comparison between the proposed CMaP-SAM method and existing approaches on the FSS-1000 dataset for both 1-shot and 5-shot segmentation tasks. The results indicated that CMaP-SAM achieved superior performance, with an mIoU of 90.2 for the 1-shot task and 90.5 for the 5-shot task. In contrast, HSNet with a ResNet-50 backbone obtained mIoU values of 85.5 and 87.8 for the 1-shot and 5-shot tasks, respectively, while VAT achieved mIoU scores of 89.5 and 90.3 for the same tasks.

\begin{table}[htbp]
  \centering
  \caption{Comparison on FSS-1000 dataset. Best in \textbf{bold}, second best in \underline{underlined}.}
  \label{tab:CMaP_fss1000_comparison}
  \scalebox{0.75}{
  \begin{tabular}{cc|c|c}
      \toprule
      \multirow{2}{*}{\textbf{Method}} & \multirow{2}{*}{\textbf{Backbone}} & \textbf{1-shot Performance} & \textbf{5-shot Performance} \\
      & & \textbf{mIoU} & \textbf{mIoU} \\
      \midrule
      HSNet~\cite{min2021hypercorrelation} & ResNet-50  & 85.5 & 87.8 \\
      VAT~\cite{HongCNLK22}             & ResNet-50  & 89.5 & 90.3 \\
      HSNet~\cite{min2021hypercorrelation} + MSI~\cite{moon2023msi}  & ResNet-50  & 87.5 & 88.4 \\
      VAT~\cite{HongCNLK22} + MSI~\cite{moon2023msi}             & ResNet-50  & 90.0 &  \underline{90.6} \\

      DACM~\cite{xiong2022doubly}    & ResNet-50  & \textbf{90.7} & \textbf{91.6} \\

      CMaP-SAM (Ours) & SAM & \underline{90.2} & 90.5 \\
      \bottomrule
  \end{tabular}
  }
\end{table}

  \subsection{Ablation Studies}

  To evaluate the effectiveness of CMaP-SAM, ablation studies were performed on the PASCAL-$5^i$ dataset (fold-0), as shown in Table~\ref{tab:CMaP_ablation_study}. The removal of the contraction mapping module resulted in decreased performance, with 1-shot and 5-shot mIoU dropping by 1.3 and 0.9, respectively, confirming that the position priors enhanced spatial localization. The adapter module was identified as the most critical component, as its removal caused the largest performance decline (5.1 for 1-shot and 4.3 for 5-shot), emphasizing the importance of feature distribution alignment in FSS-SAM integration. The refinement module's contribution became more pronounced with an increasing number of support samples, leading to a performance drop of 0.6 for 1-shot and 0.9 for 5-shot when removed. Further experiments revealed synergistic effects among the components, with the simultaneous removal of module pairs leading to disproportionately larger performance declines compared to individual removals. These findings validated the integrated design of the framework and demonstrated the complementary roles of spatial priors, feature alignment, and refinement in few-shot segmentation tasks.

  \begin{table}[t]
    \centering
    \caption{Ablation studies of the CMaP-SAM.}
    \label{tab:CMaP_ablation_study}
    \scalebox{0.85}{
    \begin{tabular}{ccc|cccc}
        \toprule
        \multicolumn{3}{c|}{\textbf{Components}} & \multicolumn{2}{c}{\textbf{1-shot Performance}} & \multicolumn{2}{c}{\textbf{5-shot Performance}} \\
        \textbf{Contraction Mapping Prior} & \textbf{Adapter} & \textbf{Refiner} & \textbf{mIoU} & \textbf{FB-IoU} & \textbf{mIoU} & \textbf{FB-IoU} \\
        \midrule
        $\checkmark$ & $\checkmark$ & $\checkmark$ & \textbf{78.1} & \textbf{88.2} & \textbf{82.2} & \textbf{90.9} \\
         & $\checkmark$ & $\checkmark$ & 76.8 & 87.4 & \underline{81.3} & 90.4 \\
        $\checkmark$ &  & $\checkmark$ & 73.0 & 84.8 & 77.9 & 88.1 \\
        $\checkmark$ & $\checkmark$ &  & \underline{77.5} & \underline{88.1} & \underline{81.3} & \underline{90.5} \\
        $\checkmark$ &  &  & 72.3 & 84.6 & 77.1 & 88.1 \\
         & $\checkmark$ &  & 74.6 & 86.2 & 79.5 & 89.4 \\
         &  & $\checkmark$ & 72.0 & 84.3 & 77.7 & 88.0 \\
        \bottomrule
    \end{tabular}}
\end{table}

\begin{table}[htbp]
  \centering
  \caption{Sensitivity analysis with respect to the number of reference images $K$ on Pascal-5$^i$ (Fold-0) and COCO-20$^i$ (Fold-0).}
  \label{tab:sensitivity_study}
  \scalebox{0.9}{
  \begin{tabular}{c|cc|cc|cc|cc}
      \toprule
      & \multicolumn{4}{c|}{\textbf{Pascal-5$^i$ (Fold-0)}} & \multicolumn{4}{c}{\textbf{COCO-20$^i$ (Fold-0)}} \\
      \cmidrule(lr){2-5} \cmidrule(lr){6-9}
      \textbf{Quantity $\mathbf{K}$}
      & \textbf{mIoU} & $\triangle$ & \textbf{FB-IoU} & $\triangle$
      & \textbf{mIoU} & $\triangle$ & \textbf{FB-IoU} & $\triangle$ \\
      \midrule
      1-shot  & 78.1 & --   & 88.2 & --   & 54.0 & --   & 73.5 & --   \\
      2-shot  & 80.3 & 2.2  & 89.8 & 1.6  & 58.2 & 4.2  & 75.8 & 2.3  \\
      3-shot  & 81.4 & 3.3  & 90.4 & 2.2  & 59.3 & 5.3  & 76.9 & 3.4  \\
      4-shot  & 82.3 & 4.2  & 90.9 & 2.7  & 60.7 & 6.7  & 77.2 & 3.7  \\
      5-shot  & 82.2 & 4.1  & 90.9 & 2.7  & 61.3 & 7.3  & 77.5 & 4.0  \\
      6-shot  & 82.2 & 4.1  & 90.9 & 2.7  & \textbf{61.8} & 7.8  & \textbf{78.3} & 4.8  \\
      7-shot  & \textbf{82.5} & 4.4 & \textbf{91.1} & 2.9 & \textbf{61.8} & 7.8  & \textbf{78.3} & 4.8  \\
      8-shot  & 82.4 & 4.3  & \textbf{91.1} & 2.9 & \textbf{61.8} & 7.8  & \textbf{78.3} & 4.8  \\
      9-shot  & \textbf{82.5} & 4.4 & \textbf{91.1} & 2.9 & \textbf{61.8} & 7.8  & \textbf{78.3} & 4.8  \\
      10-shot & \textbf{82.5} & 4.4 & \textbf{91.1} & 2.9 & \textbf{61.8} & 7.8  & \textbf{78.3} & 4.8  \\
      \bottomrule
  \end{tabular}}
\end{table}

\subsection{Sensitivity Analysis}

As shown in Table~\ref{tab:sensitivity_study}, the performance improved as the number of reference images increased from $K=1$ to $K=5$ (mIoU: 78.1 $\rightarrow$ 82.2; FB-IoU: 88.2 $\rightarrow$ 90.9), with diminishing returns observed beyond $K=5$. The peak performance was achieved at $K=7$ (mIoU = 82.5). For optimal efficiency, selecting 3-5 representative samples is recommended for practical applications.

To investigate whether the sensitivity to the number of support images varies across datasets, we extended the analysis to COCO-20\textsuperscript{\textit{i}} (Fold-0). The right part of Table~\ref{tab:sensitivity_study} shows that the performance gain per additional shot is similar to that observed on PASCAL-5\textsuperscript{\textit{i}}, confirming the method's cross-dataset insensitivity to support-set size.

Moreover, we performed another sensitivity analysis to evaluate the effect of stacking more convolution blocks on model performance. As shown in Table~\ref{tab:sensitivity_analysis_adapter_complexity}, increasing the number of blocks from 1 to 5 only yields marginal gains in mIoU (from 78.1\% to 78.4\%). This indicates that a single convolution block is sufficient for balancing performance and efficiency.

\begin{table}[htbp]
  \centering
  \caption{Sensitivity analysis with respect to the complexity of the adapter (number of stacked convolution blocks inside the adapter).}
  \begin{tabular}{c|ccccc}
  \hline
  \textbf{Number of Blocks} & 1 & 2 & 3 & 4 & 5 \\
  \hline
  \textbf{Performance (mIoU)} & 78.1 & 78.1 & 78.3 & 78.4 & 78.4 \\
  \hline
  \end{tabular}
  \label{tab:sensitivity_analysis_adapter_complexity}
  \end{table}

\subsection{Complexity analysis}

Table~\ref{tab:CMaP_complexity_analysis} presents the computational complexity analysis, including the number of parameters, multiply-accumulate operations (MACs), and inference throughput (FPS). Notably, CMaP-SAM introduces only 2.0M trainable parameters, which is significantly fewer than competing methods such as DCAMA (14.2M). However, due to inheriting SAM's large ViT-based image encoder, CMaP-SAM incurs higher total parameter count and MACs, resulting in a lower inference speed. This trade-off prioritizes segmentation accuracy over real-time efficiency. Therefore, CMaP-SAM is more suitable for offline processing scenarios, such as post-hoc analysis of medical imaging~\cite{cheng2023sam,yue2024surgicalsam}, and the real-time deployment is left for future work via lightweight backbones (e.g., EfficientSAM~\cite{xiong2024efficientsam}), pruning, or distillation.

\begin{table}[htbp]
  \centering
  \caption{Complexity analysis of CMaP-SAM and comparison with other methods.}
  \label{tab:CMaP_complexity_analysis}
  \scalebox{0.85 }{ 
    \begin{tabular}{c|cccccc}
      \toprule
      \multirow{2}{*}{\textbf{Method}}&\multicolumn{3}{c|}{\textbf{params (M)}} &\textbf{MACs}& \textbf{Speed}  \\
      &\textbf{learnable} &\textbf{frozen} &\textbf{total} &\textbf{(G)}&\textbf{(fps)}\\

      \midrule
      HSNet&2.6&25.6&28.2&\textbf{20.1}&16.0 \\
      DCAMA&14.2&25.6&39.8&39.8&\textbf{19.7}\\
      PGMA-Net&2.7&23.6&26.3&42.3&10.5\\
          \midrule
      CMaP-SAM(Ours)&\textbf{2.0}&177.2&179.2&87.6&1.9\\
      \bottomrule
    \end{tabular}
    }
  
\end{table}

\begin{figure}[htbp]
  \centering
  \includegraphics[width=0.99\textwidth]{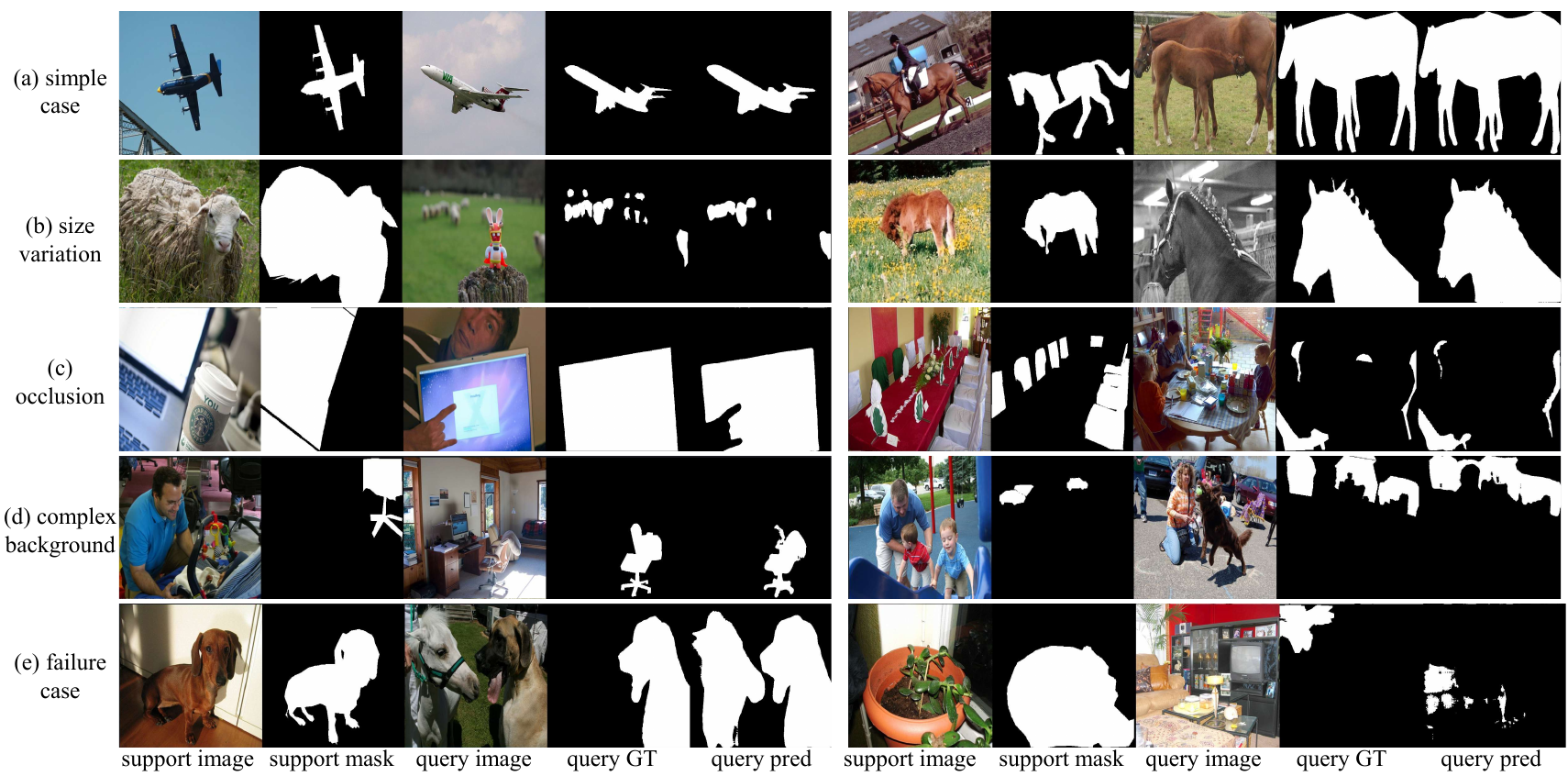}
  \caption{Visualization of segmentation results.}
  \label{fig:CMaP_segmentation_visualization}
\end{figure}

\subsection{Visualizations}

\begin{figure}[t]
  \centering
  \subfloat[]{
      \includegraphics[width=0.48\linewidth]{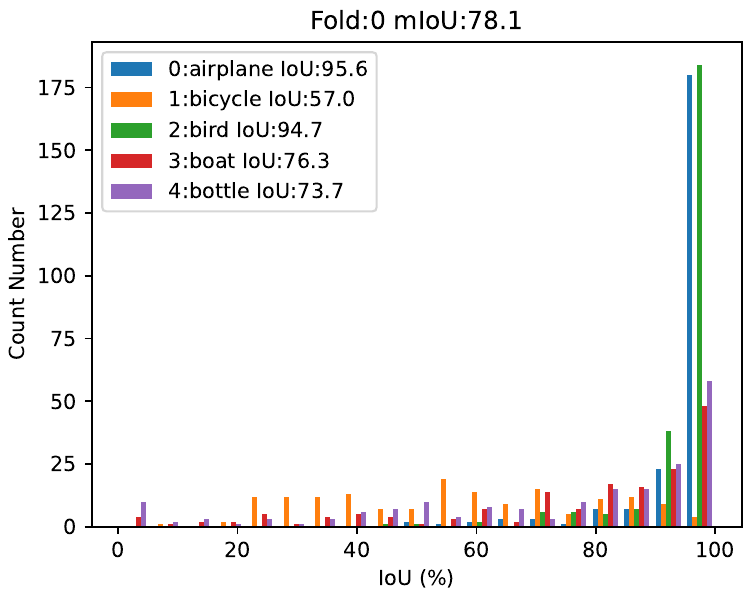}
      \label{fig:hist_fold0}
  }
  \hfill
  \subfloat[]{
      \includegraphics[width=0.48\linewidth]{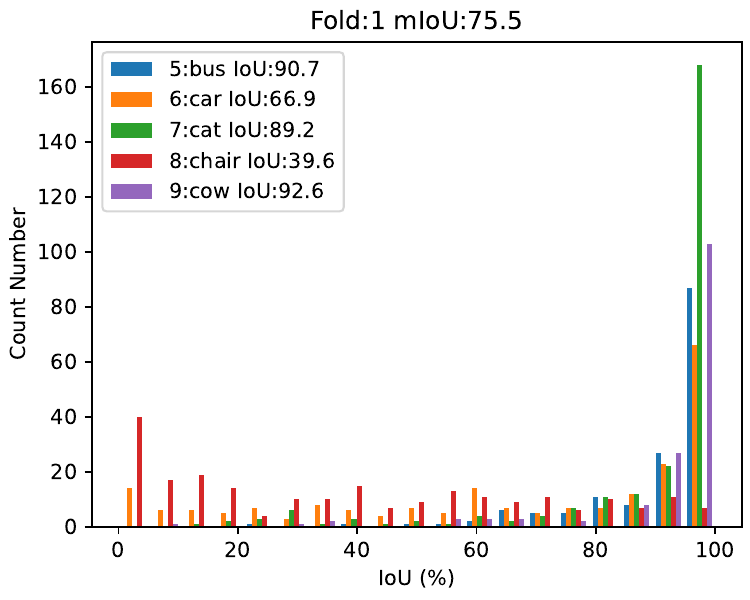}
      \label{fig:hist_fold1}
  }
  \\
  \subfloat[]{
      \includegraphics[width=0.48\linewidth]{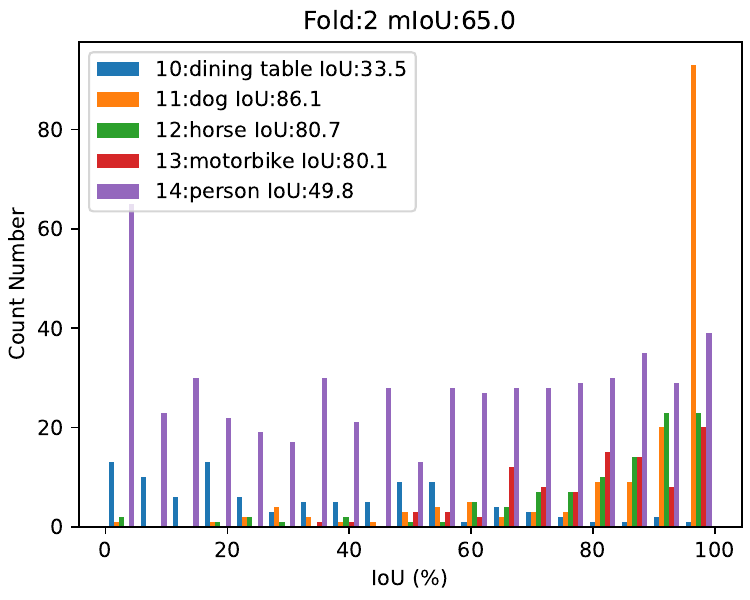}
      \label{fig:hist_fold2}
  }
  \hfill
  \subfloat[]{
      \includegraphics[width=0.48\linewidth]{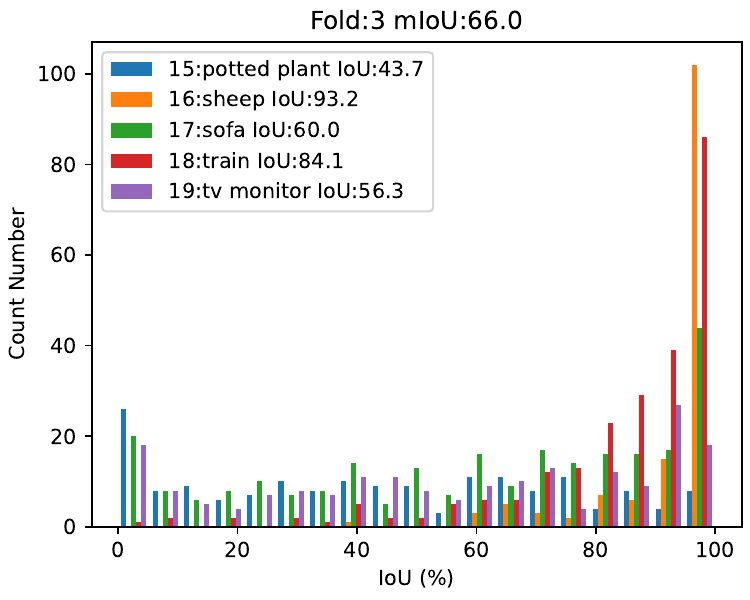}
      \label{fig:hist_fold3}
  }
  \caption{Class-wise performance distribution histograms of CMaP-SAM on PASCAL-$5^i$ dataset. The x-axis represents the mIoU values, while the y-axis indicates the number of classes. The red dashed line denotes the average mIoU value.}
  \label{fig:CMaP_classwise_performance}
\end{figure}

\textbf{Visualization of Segmentation Results:} 
Fig.~\ref{fig:CMaP_segmentation_visualization} illustrates segmentation results on the PASCAL-$5^i$ and COCO-$20^i$ datasets. The method demonstrated precise segmentation of object boundaries in simple cases (row 1), effective handling of sheep with significant scale variations and grayscale horse images despite limited color information (row 2), and robust performance in occlusion scenarios where occluded chair parts were accurately segmented while producing semantically superior results compared to the ground truth (e.g., correctly excluding erroneously labeled hand regions) (row 3). Notably, even in complex backgrounds, the method maintained accurate segmentation of chairs and cars (row 4). These visualizations validate the effectiveness of the contraction mapping theory-based position prior optimization framework across diverse and challenging scenarios. However, as shown in row 5, the method fails when the foreground and background objects in the query image exhibit similar appearances, leading to confusion. Furthermore, significant disparities between support and query images, such as those observed for plants with highly dissimilar appearances, adversely impact performance.

\textbf{Class-wise Performance Distribution:} Fig.~\ref{fig:CMaP_classwise_performance} presents performance histograms across semantic categories for all four folds of the PASCAL-$5^i$ dataset. The distributions exhibited notable right-skewness, with performance metrics concentrated in the medium-to-high value regions, forming dense clusters within high-performance intervals. This statistical pattern confirmed the method's ability to achieve stable and superior segmentation results across the majority of semantic categories.

\begin{figure}[htbp]
  \centering
  \includegraphics[width=1\textwidth]{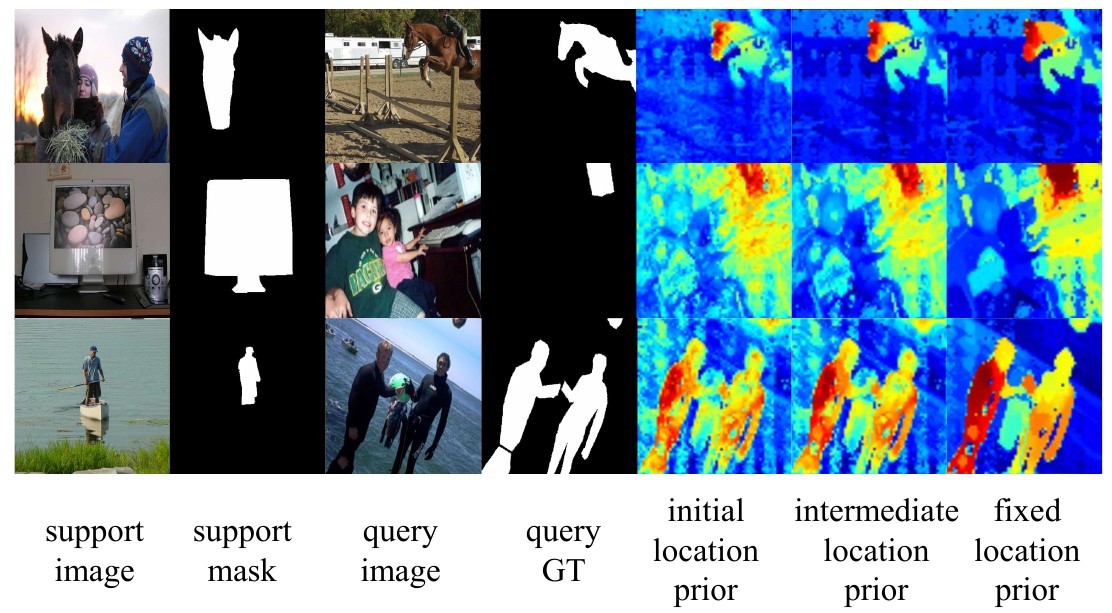}
  \caption{Visualization of the contraction mapping prior optimization process.}
  \label{fig:CMaP_fixpoint_visualization}
\end{figure}

\textbf{Visualization of contraction mapping prior optimization:} Fig.~\ref{fig:CMaP_fixpoint_visualization} illustrates the contraction mapping optimization process within the CMaP-SAM framework. The sequence includes support and query images with their corresponding masks, the initial position prior, intermediate states, and the final converged result. The initial prior exhibited coarse localization with substantial noise. As contraction mapping iterations progressed, a clear convergence pattern emerged, with attention regions becoming increasingly concentrated. The final prior achieved robust target localization, effectively suppressing background interference while preserving structural integrity.

\subsection{Limitations and Future Work}
Despite the promising performance of CMaP-SAM, several limitations remain. First, the model inherits the computational overhead from SAM, as it retains ViT-based image encoder, resulting in higher FLOPs compared to more lightweight CNN-based alternatives. Second, a performance saturation phenomenon is observed in our experiments: as the number of support samples increases, the improvement in performance gradually plateaus, which remains a critical challenge in FSS field. These limitations suggest directions for future work, including the exploration of lightweight SAM variants (e.g., EfficientSAM~\cite{xiong2024efficientsam}), pruning, and model distillation techniques to reduce inference latency while maintaining high accuracy. Additionally, developing meta-learned prompt generators could further improve efficiency and scalability for real-world applications.

\section{Conclusion}
This paper introduced CMaP-SAM, a framework that effectively bridges the gap between the Segment Anything Model and few-shot segmentation tasks through three components. First, position prior optimization was mathematically reformulated as a Banach contraction mapping, providing convergence guarantees while preserving both semantic information and structural correlations. Building on this foundation, an adaptive distribution alignment module was developed to connect continuous probability distributions with SAM's discrete prompt requirements, fundamentally addressing the information loss prevalent in prior approaches. These components work synergistically with a foreground-background decoupled refinement architecture to significantly enhance segmentation accuracy. Experimental validation confirmed the effectiveness of this approach, demonstrating 71.1 mIoU on PASCAL-$5^i$ and 56.1 mIoU on COCO-$20^i$ benchmarks. Beyond these empirical results, CMaP-SAM contributes valuable theoretical insights through its provable convergence properties, with potential applications in domains where annotation costs remain prohibitive yet precise segmentation is essential.

\section*{Acknowledgements}
This work was supported in part by National Science and Technology Major Project (2021ZD0112001), National Natural Science Foundation of China (Grant 62271119), Natural Science Foundation of Sichuan Province (2023NSFSC1972).

\end{document}